\documentclass[lettersize,journal]{IEEEtran}
\usepackage{amsmath,amsfonts}
\usepackage{algorithmic}
\usepackage{algorithm}
\usepackage{array}
\usepackage[caption=false,font=normalsize,labelfont=sf,textfont=sf]{subfig}
\usepackage{textcomp}
\usepackage{stfloats}
\usepackage{url}
\usepackage{verbatim}
\usepackage{graphicx}
\usepackage{cite}
\usepackage{multirow}
\usepackage{amsthm}

\newtheorem{thm}{Theorem}
\newtheorem{defn}{Definition}

\newtheorem{rem}{Remark}

\begin{document}

\title{Robot Safe Planning In Dynamic Environments Based On Model Predictive Control Using Control Barrier Function}

\author{Zetao Lu, Kaijun Feng, Jun Xu,~\IEEEmembership{Senior Member,~IEEE}, Haoyao Chen,~\IEEEmembership{Member,~IEEE} and Yunjiang Lou,~\IEEEmembership{Senior Member,~IEEE}
\thanks{This work was supported in part by the National Natural Science Foundation of China under Grant 62173113, and in part by the Science and Technology Innovation Committee of Shenzhen Municipality under Grant GXWD20231129101652001, and in part by Natural Science Foundation of Guangdong Province of China under Grant 2022A1515011584. (\textit{Corresponding author: Jun Xu.})}
\thanks{The authors are with the School of Mechanical Engineering and Automation, Harbin Institute of Technology Shenzhen, Shenzhen, Guangdong, China, 518055 (\textit{email: \{22S153114, 21S153137\}@stu.hit.edu.cn; \{xujunqgy, hychen5, louyj\}@hit.edu.cn}).}}

\markboth{Journal of \LaTeX\ Class Files,~Vol.~14, No.~8, August~2021}%
{Shell \MakeLowercase{\textit{et al.}}: A Sample Article Using IEEEtran.cls for IEEE Journals}


\maketitle

\begin{abstract}
Implementing obstacle avoidance in dynamic environments is a challenging problem for robots. Model predictive control (MPC) is a popular strategy for dealing with this type of problem, and recent work mainly uses control barrier function (CBF) as hard constraints to ensure that the system state remains in the safe set. However, in crowded scenarios, effective solutions may not be obtained due to infeasibility problems, resulting in degraded controller performance. We propose a new MPC framework that integrates CBF to tackle the issue of obstacle avoidance in dynamic environments, in which the infeasibility problem induced by hard constraints operating over the whole prediction horizon is solved by softening the constraints and introducing exact penalty, prompting the robot to actively seek out new paths. At the same time, generalized CBF is extended as a single-step safety constraint of the controller to enhance the safety of the robot during navigation. The efficacy of the proposed method is first shown through simulation experiments, in which a double-integrator system and a unicycle system are employed, and the proposed method outperforms other controllers in terms of safety, feasibility, and navigation efficiency. Furthermore, real-world experiment on an MR1000 robot is implemented to demonstrate the effectiveness of the proposed method.
\end{abstract}


\section{Introduction}
\IEEEPARstart{I}{n} recent years, with the continuous development of robot technology, the application scope of robots is no longer limited to industrial manufacturing scenarios, but has also been expanding to other industries, such as autonomous driving, inspection robots, disinfection robots and delivery robots \cite{lee2023survey}. In order to achieve safe navigation of robots in dynamic and shared environments, it is of great significance to design a safety-critical controller to enable the autonomous system to achieve optimal performance while ensuring safety. Some recent work combines the control barrier function (CBF) with model predictive control (MPC) \cite{zeng2021safety} to implement such a safety-critical controller and applies it to dynamic environments by extending CBF to dynamic CBF (D-CBF) \cite{jian2023dynamic}. However, due to the existence of state constraints, applying CBF as hard constraints to the entire prediction horizon may lead to failure in solving the optimization problem, which is particularly obvious in complex dynamic environments.

In order to solve the aforementioned problems and achieve better control effects, we transformed the CBF hard constraints into soft constraints and incorporated them into the penalty function of the optimization problem. Besides, a single-step CBF is imposed to enhance safety. Through our approach, the robot is able to significantly reduce the probability of solution failure in dynamic environments and reach its destination with higher efficiency and safety.

\subsection{Related Work}
Existing robot navigation work in dynamic environments can be divided into three categories: 1) reactive based; 2) learning based; 3) optimization based. In reactive-based approaches, the robot makes one-step optimal action based on information about dynamic obstacles in the current environment, including velocity obstacle (VO) \cite{fiorini1998motion} and its variants \cite{van2011reciprocal, gonon2021reactive}. However, these types of methods usually do not take into account the robot's kinematic constraints. Relying solely on current state information can result in short-term, oscillatory, and unnatural behavior, which does not facilitate pedestrian understanding of the robot's movement intentions\cite{chen2021interactive}. In learning-based approaches, robots endeavor to emulate appropriate navigation strategies. By imitating the interactive movements of pedestrians, they aim to navigate through dense crowds in a manner that is more socially acceptable. Deep reinforcement learning is often used to train computationally efficient navigation strategies \cite{chen2019crowd, yang2023rmrl}, which implicitly encodes interactions and collaborations between pedestrians to generate paths with behavioral patterns more consistent with humans \cite{boldrer2022multi}. However, learning-based methods are dependent on offline training and are then restricted by environmental characteristics, which may encounter generalization issues when transitioning from simulation to real world, i.e., the performance in scenarios not covered by the training data can not be guaranteed. Optimization-based methods usually consist of two consecutive steps of prediction and planning at each time step, first using a motion model to predict dynamic obstacles in the environment \cite{lindqvist2020nonlinear}, and then formulating robot navigation as an optimal control problem \cite{chen2021interactive}. Such methods are usually based on MPC, as it is able to integrate the kinematic constraints and static/dynamic collision constraints of the robot while combining planning with control to find an ideal trajectory \cite{brito2019model}.

However, a significant concern with optimization-based methods in dynamic environments is the safety of the generated trajectories. CBF has recently been introduced as an effective method, combined with MPC \cite{zeng2021safety}, to design safety-critical controllers that can guarantee effective safety margins under a short prediction horizon. In dynamic environments, \cite{jian2023dynamic} implemented obstacle avoidance with a safety-critical controller built based on lidar and dynamic CBF. In a static maze scenario, \cite{thirugnanam2022safety} successfully navigated different robot shapes using relaxation technology \cite{zeng2021enhancing}. However, applying CBF as hard constraints to the entire prediction horizon may lead to failure in solving the optimization problem \cite{zeng2021safety}. \cite{ma2021feasibility} proposed generalized CBF (GCBF) to use CBF constraint as a one-step constraint to improve feasibility, but there is still a trade-off between feasibility and safety. In \cite{zeng2021enhancing}, the trade-off is handled by incorporating slack variables into the CBF constraints to enhance feasibility, although this approach inherently increases the solution time and diminishes the safety margin.

Inspired by these studies, we contemplate the conversion of CBF hard constraints into soft constraints. The goal is to maintain control effects that are comparable to those of hard constraints while minimizing the likelihood of solution failure. This approach inspires robots to actively seek feasible paths in dynamic environments. Concurrently, it's essential to introduce an effective safety guarantee, fulfilled by integrating dynamic GCBF (D-GCBF) constraint. The paper's main focus lies on the application of CBF within the framework of MPC, aiming to enable robots to navigate through crowded and complex dynamic environments efficiently while ensuring safety.

\subsection{Contribution}
The contributions of this paper are as follows.
\begin{itemize}
  \item We propose an MPC framework based on CBF soft constraints for generating safe collision-free trajectories in dynamic environments;
  \item We incorporate D-GCBF within this framework as a single-step hard constraint to enhance safety;
  \item Simulation experiments and real-world tests were carried out to validate the real-time capability, effectiveness, and stability of the algorithm.
\end{itemize}

\subsection{Paper Structure}
This paper is structured as follows. In Section \ref{sec2}, we provide an overview of the definition of the CBF and its associated optimization problem construction when used as hard constraints. In Section \ref{sec3}, we transform CBF hard constraints into soft constraints and derive conditions for exact penalty. Besides, single-step D-GCBF hard constraint is imposed as safety guarantee. In order to verify the effectiveness of the controller design and algorithm, examples of obstacle avoidance of our algorithm in the simulation environment and the real world are demonstrated in Section \ref{sec4}. Section \ref{sec5} concludes the paper.

\section{PRELIMINARIES} \label{sec2}
In this section, we present the preliminaries related to the CBF and propose the basic form of the optimization problem based on MPC and CBF. This lays the foundation for subsequent controller design.
\subsection{Problem Formulation}
Consider the robot's motion model as a discrete-time control system
\begin{equation}
    \label{eq:rbsys}
\mathbf{x}_{k+1} = f(\mathbf{x}_k, \mathbf{u}_k),
\end{equation}
where $\mathbf{x}_k\in\mathcal{X}\subset\mathbb{R}^n$ is the state of the system, $\mathbf{u}_k\in\mathcal{U}\subset\mathbb{R}^m$ is the control input. In a dynamic environment, assuming that the motion equation of a moving obstacle $\mathbf{o}^i$ is
\begin{equation}
    \label{eq:mosys}
\mathbf{o}^i_{k+1} = \xi(\mathbf{o}_k),
\end{equation}
where $\mathbf{o}^i_k\in\mathbb{R}^{n_o}$ represents the state of the moving obstacle at time $k$, the superscript $i\in\{1,2,\dots,N_o\}$ represents the $i$-th moving obstacle, $\xi (\cdot)$ is the state transition function. In robot obstacle avoidance scenarios, the obstacle avoidance problem is usually described using an optimal control problem based on distance constraints\cite{zhang2020optimization}. Assuming that the robot and the moving obstacles are approximated by circles with center points $(x_{t},y_{t})$ and $(x^{i}_{t},y^{i}_{t})$, radii $r_r$ and $r_{\mathbf{o}^i}$ respectively on the two-dimensional plane, then the safe distance between the robot and the moving obstacle $\mathbf{o}^i$ is defined as $r_i = r_r + r_{\mathbf{o}^i} + \epsilon$, where $\epsilon$ is the additional safety margin. So at time step $t$, the  MPC problem based on distance constraints (MPC-DC) is as follows
\begin{subequations}
    \label{eq:mpcdc}
    \begin{align}
        \min_{\mathbf{u}_{t+k|t}}\ & p(\mathbf{x}_{t+N|t})+\sum^{N-1}_{k=0}q(\mathbf{x}_{t+k|t},\mathbf{u}_{t+k|t}) \label{mpcdc:a}\\
        \mathrm{s.t.}\ &\mathrm{for}\ \mathrm{all}\ k=0,\dots,N-1: \notag\\
        &\mathbf{x}_{t+k+1|t}=f(\mathbf{x}_{t+k|t},\mathbf{u}_{t+k|t}), \label{mpcdc:b}\\
        &\mathbf{x}_{t+k|t}\in\mathcal{X}, \label{mpcdc:c}\\
        &\mathbf{u}_{t+k|t}\in\mathcal{U}, \label{mpcdc:d}\\
        &\mathbf{x}_{t|t}=\mathbf{x}_t, \label{mpcdc:e}\\
        &g_i(\mathbf{x}_{t+k|t},\mathbf{o}^i_{t+k|t})\geq0, \label{mpcdc:f}
    \end{align}
\end{subequations}
where $N$ is the prediction horizon. The vectors $\mathbf{x}_{t+k|t}$ and $\mathbf{u}_{t+k|t}$ represent the predicted state and designed input at time step $t + k$, respectively. 
The first term of the cost function (\ref{mpcdc:a}) is the terminal cost, and the latter one is the stage cost. (\ref{mpcdc:c}) and (\ref{mpcdc:d}) represent the state constraints and input constraints along the prediction horizon, respectively. The safety limit distance constraints are represented by (\ref{mpcdc:f}), where $g_i(\mathbf{x}_{t+k|t},\mathbf{o}^i_{t+k|t})=\sqrt{(x_{t+k|t}-x^{i}_{t+k|t})^2+(y_{t+k|t}-y^{i}_{t+k|t})^2}-r_i$.

The optimal solution of this problem at time step $t$ is an input sequence i.e. $\{\mathbf{u}^*_{t|t},\dots, \mathbf{u}^*_{t+N-1|t}\}$. Only the first element $\mathbf{u}^*_{t|t}$ of the optimal solution will become the control input of the system (\ref{eq:rbsys}). Then the above optimization problem is solved repeatedly at the new state $\mathbf{x}_{t+1}$.

\subsection{Control Barrier Function}
In control theory, CBF is a continuously differentiable function used to ensure forward invariance of the system state. When the system state is on the boundary of the invariant set, CBF can adjust the input of the control system to keep it within the invariant set. The definition of CBF is given below based on the concept of safety set in \cite{ames2019control}. Assume that the safe set $\mathcal{C}$ is a super-level set of a continuously differentiable function $h:\mathbb{R}^n\rightarrow\mathbb{R}$:
\begin{subequations}
	\label{eq:ss}
    \begin{align}
        \mathcal{C}&=\{\mathbf{x}\in\mathbb{R}^n:h(\mathbf{x})\geq0\}, \label{ss:a} \\
        \partial \mathcal{C}&=\{\mathbf{x}\in\mathbb{R}^n:h(\mathbf{x})=0\}, \label{ss:b} \\
        \mathrm{Int}(\mathcal{C})&=\{\mathbf{x}\in\mathbb{R}^n:h(\mathbf{x})>0\}. \label{ss:c}
    \end{align}
\end{subequations}
And $\frac{\partial h(\mathbf{x})}{\partial \mathbf{x}}\neq0$ holds for all points on the boundary of the safe set. Then if and only if $\dot h(\mathbf{x})=\frac{\partial h(\mathbf{x})}{\partial \mathbf{x}}\dot{\mathbf{x}}\geq0,\forall \mathbf{x}\in\partial\mathcal{C}$, the set $\mathcal{C}$ is a forward invariant set, that is a safe set.
\begin{defn}[Discrete-time CBF\cite{zeng2021safety}]
Consider the discrete-time system ($\ref{eq:rbsys}$). Given a set $\mathcal{C}$ defined by ($\ref{eq:ss}$) for a function $h:\mathbb{R}^n\rightarrow\mathbb{R}$, the function $h$ is a discrete-time CBF if there exists a function $\gamma\in\mathcal{K}_{\infty}$ s.t.
\begin{equation}
    \label{eq:hx1}
\Delta h(\mathbf{x}_k,\mathbf{u}_k)\geq-\gamma (h(\mathbf{x}_k)),
\end{equation}
where $\Delta h(\mathbf{x}_k,\mathbf{u}_k):=h(\mathbf{x}_{k+1})-h(\mathbf{x}_{k})$.
\end{defn}

When discrete-time CBF is used as constraints in the safety-critical MPC design, the safety of the system can be fully guaranteed while avoiding static obstacles\cite{zeng2021safety}. In order to better apply it to dynamic scenes, \cite{jian2023dynamic} proposed D-CBF on this basis. Similarly, we assume that the shape of the moving obstacle does not change and modify (\ref{eq:hx1}) to the following form
\begin{equation}
    \label{eq:hx2}
\Delta h_i(\mathbf{x}_k,\mathbf{u}_k,\mathbf{o}^i_k)\geq-\gamma (h_i(\mathbf{x}_k,\mathbf{o}^i_k)),
\end{equation}
where $\Delta h_i(\mathbf{x}_k,\mathbf{u}_k,\mathbf{o}^i_k):=h(\mathbf{x}_{k+1},\mathbf{o}^i_{k+1})-h(\mathbf{x}_{k},\mathbf{o}^i_k)$. We follow the result of \cite{zeng2021safety} and select the $\mathcal{K}_{\infty}$ function $\gamma(\cdot)$ as a constant $\gamma\in(0,1]$. 

Therefore, assuming that the states of the robot and the moving obstacles at time $t$ are known, the MPC problem based on D-CBF (\ref{eq:hx2}) constraints (MPC-D-CBF) is as follows
\begin{subequations}
    \label{eq:mpcdcbf}
    \begin{align}
        \min_{\mathbf{u}_{t+k|t}}\ &p(\mathbf{x}_{t+N|t})+\sum^{N-1}_{k=0}q(\mathbf{x}_{t+k|t},\mathbf{u}_{t+k|t}) \label{mpcdcbf:a}\\
        \mathrm{s.t.}\ &\mathrm{for}\ \mathrm{all}\ k=0,\dots,N-1: \notag\\
        &\mathbf{x}_{t+k+1|t}=f(\mathbf{x}_{t+k|t}, \mathbf{u}_{t+k|t}), \label{mpcdcbf:b}\\
        &\mathbf{x}_{t+k|t}\in\mathcal{X}, \label{mpcdcbf:c}\\
        &\mathbf{u}_{t+k|t}\in\mathcal{U}, \label{mpcdcbf:d}\\
        &\mathbf{x}_{t|t}=\mathbf{x}_t, \label{mpcdcbf:e}\\
        &h_i(\mathbf{x}_{t+k+1|t},\mathbf{o}^i_{t+k+1|t})\geq(1-\gamma)h_i(\mathbf{x}_{t+k|t},\mathbf{o}^i_{t+k|t}), \label{mpcdcbf:f}
    \end{align}
\end{subequations}
where the constraint (\ref{mpcdcbf:f}) guarantees the forward invariance of the safe set $\mathcal{C}$ \cite{zeng2021safety}, and $\mathcal{C}$ is defined in (4). When $\gamma=1$, constraint (\ref{mpcdcbf:f}) degenerates into constraint (\ref{mpcdc:f}), which will lead to a decrease in safety. However, the smaller the value of $\gamma$ is, the more stringent the constraints will become, which makes the optimization problem difficult to solve or even unsolvable.

\begin{figure}[!t]
\centering
\includegraphics[width=0.45\textwidth]{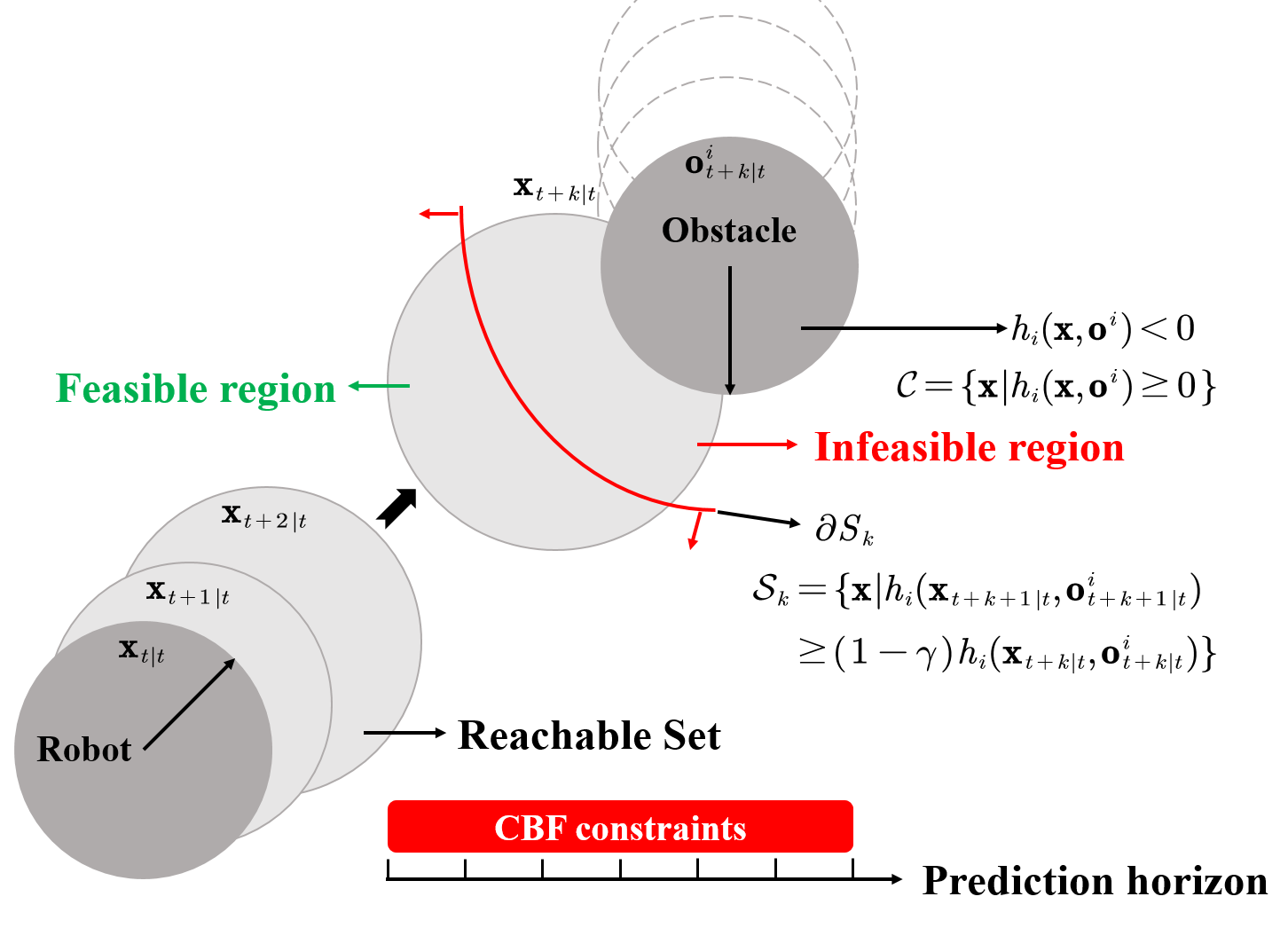}
\caption{Feasibility of optimization problem (\ref{eq:mpcdcbf}). The reachable set propagates along the prediction horizon, starting from the initial state $\mathbf{x}_{t|t}$. The definition of the safety set $\mathcal{C}$ is derived from equation (\ref{eq:ss}). $\mathcal{S}$ represents the state space set that satisfies constraints (\ref{mpcdcbf:c}) and (\ref{mpcdcbf:f}). The red arc represents the boundary of $\mathcal{S}$, and its interior is depicted by a red arrow. The optimization problem is feasible only if the intersection of the reachable set and $\mathcal{S}$ is not empty.}
\label{cbf}
\end{figure}

The CBF hard constraints (\ref{mpcdcbf:f}) in the optimization problem act on the entire prediction horizon. Thus, if there is no feasible region in any step of the prediction horizon, it renders the entire optimization problem infeasible. As seen in Fig. \ref{cbf}, when the CBF constraints are satisfied, the system state must be within the safe set $\mathcal{C}$. However, violating the CBF constraints does not imply that the system state is outside the safe set $\mathcal{C}$. Specifically, even if there is no feasible region at step $k$, it does not mean that the current state lacks a suitable input to avoid collision. Although the optimization problem (\ref{eq:mpcdcbf}) proves efficient and safe in simple scenarios, its performance may degrade in crowded environments. The reason is the application of CBF hard constraints (\ref{mpcdcbf:f}) over the entire prediction horizon could lead to frequent solution failures.

\section{CONTROLLER DESIGN} \label{sec3}
When the optimization problem (\ref{eq:mpcdcbf}) is infeasible, optimal control cannot be obtained, which may reduce the navigation efficiency of the mobile robot in obstacle avoidance scenarios. Hence, in the following, we formulate the relaxed safety control logic to alleviate this problem.
\subsection{Soft Constrained Predictive Control With CBF}\label{sec:soft_constrained}
According to \cite{zeng2021safety}, the infeasibility problem encountered in MPC arises from the intersection of the reachable set and set $\mathcal{S}$, which satisfies CBF constraints at horizon step $k$, being empty. In dynamic scenarios, this problem is especially serious as the set $\mathcal{S}$ is mainly determined by dynamic obstacles and the initial state of the robot, which is dynamically changing. The mobile robot is then more likely to enter a state that is infeasible, so new control inputs cannot be obtained by solving (\ref{eq:mpcdcbf}). In this case, methods such as repeating the control input from the previous moment or calculating the control input without constraints will violate the controller requirements and may lead to unpredictable and dangerous behavior. The most direct way to solve this problem is to adjust the prediction horizon $N$, but this will also change the prediction ability of the controller. When a robot's ability to predict is not adequate, it tends to explore riskier areas more often. This behaviour, in turn, worsens its initial state at future time instants.

In the optimization problem (\ref{eq:mpcdcbf}), the scalar $\gamma$ in the state constraints (\ref{mpcdcbf:e}) is also called a conservative coefficient. This is because we can find a trade-off between safety and feasibility by choosing the value of $\gamma$. However, gaining feasibility by reducing safety is not our original intention. Some state constraint softening methods for MPC have been proposed, such as \cite{di1989exact,kerrigan2000soft}, which can ensure the feasibility of online optimization problems under unexpected disturbances. Inspired by these studies, we modify (\ref{eq:mpcdcbf}) to the soft-constrained  MPC problem based on CBF (SCMPC-CBF)
\begin{subequations}
    \label{eq:our}
    \begin{align}
        \min_{\mathbf{u}_{t+k|t},\mathbf{\zeta}_{t+k|t}}\ &p(\mathbf{x}_{t+N|t})+\sum^{N-1}_{k=0}q(\mathbf{x}_{t+k|t},\mathbf{u}_{t+k|t}) \notag \\
        &+\alpha \sum^{N-1}_{k=0}\|\mathbf{\zeta}_{t+k|t}\| \label{our:a}\\
        \mathrm{s.t.}\ &\mathrm{for}\ \mathrm{all}\ k=0,\dots,N-1: \notag\\
        &\mathbf{x}_{t+k+1|t}=f(\mathbf{x}_{t+k|t},\mathbf{u}_{t+k|t}), \label{our:b}\\
        &\mathbf{x}_{t+k|t}\in\mathcal{X}, \label{our:c}\\
        &\mathbf{u}_{t+k|t}\in\mathcal{U}, \label{our:d}\\
        &\mathbf{x}_{t|t}=\mathbf{x}_t, \label{our:e}\\
        &\mathbf{x}_{t+k|t}\in\mathcal{X}(\mathbf{\zeta}_{t+k|t}),\mathbf{\zeta}_{t+k|t}\geq0, \label{our:f}
    \end{align}
\end{subequations}
where $\mathbf{\zeta}\in\mathbb{R}^{N_o}$ is the slack variable and can be expressed as
$\mathbf{\zeta}_{t+k|t}=[\zeta_{t+k|t}^1, \ldots, \zeta_{t+k|t}^{N_o}]^T$. 
The soft constraint (\ref{our:f}) can be then described as,
\begin{equation}
\begin{array}{rl}
\mathcal{X}(\zeta_{t+k|t})&=\{\mathbf{x}\in \mathbb{R}^n|\zeta_{t+k|t}^i\geq(1-\gamma)h_i(\mathbf{x}_{t+k|t}, \mathbf{o}_{t+k|t}^i)\\
&-h_i(\mathbf{x}_{t+k+1|t}, \mathbf{o}_{t+k+1|t}^i), \forall i=1,\ldots, N_o\}.
\end{array}
\end{equation}
In (\ref{our:a}), $\alpha$ is the constraint violation penalty weight. This ensures that the optimization problem is feasible for any input sequence in $\mathcal{U}$. Even if $\mathbf{x}_{t+k|t}$ does not satisfy the constraint (\ref{mpcdcbf:f}), these violations are penalized in the cost function to determine the value of the slack variables. By constructing in this manner, we can provide an optimal solution that is not only close to that of (\ref{eq:mpcdcbf}) but also ensures feasibility.

\begin{thm}
Given a state $\mathbf{x}_{t}$, if $\mathbf{u}_{t+k|t}^*, k=0, \ldots, N-1$ is the optimal solution to (\ref{eq:mpcdcbf}), then there exists a Lagrange vector $\mathbf{\lambda}^*$ such that
\[
\mathcal{L}(\mathbf{u}_{t+k|t}^*, \mathbf{\lambda}^*)=0,
\]
in which $\mathcal{L}(\mathbf{u}_{t+k|t}, \mathbf{\lambda})$ is the Lagrangian of the optimization problem (\ref{eq:mpcdcbf}), i.e.,
\begin{equation}\label{eq:Lagranian}
\begin{array}{rl}
\mathcal{L}(\mathbf{u}_{t+k}, \mathbf{\lambda})&=J_t(\mathbf{x}_t,\mathbf{u}_{t|t}, \ldots, \mathbf{u}_{t+N-1|t})\\
{}&+\mathbf{\lambda}^T c_t(\mathbf{x}_t,\mathbf{u}_{t|t}, \ldots, \mathbf{u}_{t+N-1|t}).
\end{array}
\end{equation}
In (\ref{eq:Lagranian}), the expressions $J_t(\mathbf{x}_t,\mathbf{u}_{t|t}, \ldots, \mathbf{u}_{t+N-1|t})$ and $c_t(\mathbf{x}_t,\mathbf{u}_{t|t}, \ldots, \mathbf{u}_{t+N-1|t})$ are the cost and constraint corresponding to (\ref{eq:mpcdcbf}), respectively.
Besides, if we choose the penalty $\alpha$ such that \[\alpha>\|\mathbf{\lambda}^*\|_D,\]
where $\|\cdot \|_D$ denotes the dual norm with respect to the norm for $\mathbf{\zeta}_{t+k|t}$ in (\ref{our:a})
and $\mathbf{u}_{t+k|t}^*$ satisfies (\ref{mpcdcbf:b})-(\ref{mpcdcbf:f}), then the optimal solutions to (\ref{eq:mpcdcbf}) and (\ref{eq:our}) are equivalent.
\end{thm}
\begin{proof}[Proof]
See \cite{fletcher2000practical} Thm 14.2.1, 14.3.1.
\end{proof}

\begin{rem}
    It is noted that if we choose the penalty weight $\alpha$ such that
\[
\alpha>\max\limits_{\mathbf{x}_t}\|\lambda^*\|_{D},
\]
then for all initial states $\mathbf{x}_t$, if the optimization problem (\ref{eq:mpcdcbf}) is feasible, then the optimal solution of (\ref{eq:mpcdcbf}) and (\ref{eq:our}) are equivalent. 

In practice, it is not easy to determine $\max_{\mathbf{x}_t}\|\lambda^*\|_D$, as the terms $J_t, c_t$ in the Lagrangian (\ref{eq:Lagranian}) change with respect to the state $\mathbf{x}_t$. Hence, 
before designing a soft constrained predictive control problem, we select a number of initial states $\mathbf{x}_t$ to estimate different $\alpha$ and then choose the maximum one as the penalty factor.
\end{rem}

After the softening, if the optimization problem (\ref{eq:mpcdcbf}) is infeasible, we can use the optimization problem (\ref{eq:our}) to find a solution. However, this solution may not necessarily be feasible for (\ref{eq:mpcdcbf}).

\subsection{Safety Enhancement with D-GCBF}
In Section \ref{sec:soft_constrained}, the infeasibility problem of the obstacle avoidance problem (\ref{eq:mpcdcbf}) is solved by softening the CBF hard constraints. In this way, the softened problem (\ref{eq:our}) is always feasible. However, the solutions of (\ref{eq:our}) and (\ref{eq:mpcdcbf}) are equivalent only when (\ref{eq:mpcdcbf}) is feasible. In safety-critical situations, conflicting with obstacles is not desirable, and further efforts should be devoted to enhancing safety in SCMPC-CBF.


Common methods for ensuring safety are to impose a control invariant set \cite{anevlavis2021controlled}, and in our previous work \cite{10342054}, by constructing a safety filter \cite{tearle2021predictive}, which is basically a control invariant set, the safety of reinforcement-learning generated controller can be improved. Here, we adopt similar ideas of the safety filter, and considering that control invariant sets are difficult to calculate for high-dimensional nonlinear systems, the dynamic generalized CBF (D-GCBF) is employed.
In order to design the D-GCBF, 
the relative degree of the state constraint to the system (\ref{our:b}) is first considered.

\begin{defn}[Relative-degree \cite{sun2003initial}]
    The state constraint $h(\mathbf{x}_t)$ of system ($\ref{eq:rbsys}$) has relative-degree $d$ with respect to control input $\mathbf{u}_t$ if 
\begin{equation}
\frac{\partial h(\mathbf{x}_{t+j})}{\partial \mathbf{u}_t}=0,\ \frac{\partial h(\mathbf{x}_{t+d})}{\partial \mathbf{u}_t}\neq0,
\end{equation}
for $\forall j\in\{0,1,\dots,d-1\}$, $\forall \mathbf{x}\in\mathbb{R}^{n}$.
\end{defn}
That is, the relative degree $d$ is the delay step at which the control input $\mathbf{u}_t$ appears in $\mathbf{y}_t$. Therefore, it is valid to impose safety constraints at time step $d$ but not at time step $j$.

By incorporating one-step state constraint at time step $d$, the optimization problem can benefit from a wider feasible region and improved computational efficiency, as opposed to including state constraints across $d$ steps \cite{ma2021feasibility}. We proposed D-GCBF in dynamic scenarios based on the results of \cite{ma2021feasibility} and the definition of relative degree.
\begin{defn}[D-GCBF]
    Consider the discrete-time system ($\ref{eq:rbsys}$). Given a set $\mathcal{C}$ defined by ($\ref{eq:ss}$)  for a function $h:\mathbb{R}^{n}\times\mathbb{R}^{n_o}\rightarrow\mathbb{R}$, the function $h$ is a dynamic generalized  CBF if
\begin{equation}
    \label{eq:hx3}
h_i(\mathbf{x}_{t+d},\mathbf{o}^i_{t+d})\geq(1-\eta)^dh_i(\mathbf{x}_t,\mathbf{o}^i_t),
\end{equation}
where the constant $\eta\in(\gamma,1]$. 
\end{defn}

\begin{figure}[!t]
\centering
\includegraphics[width=0.35\textwidth]{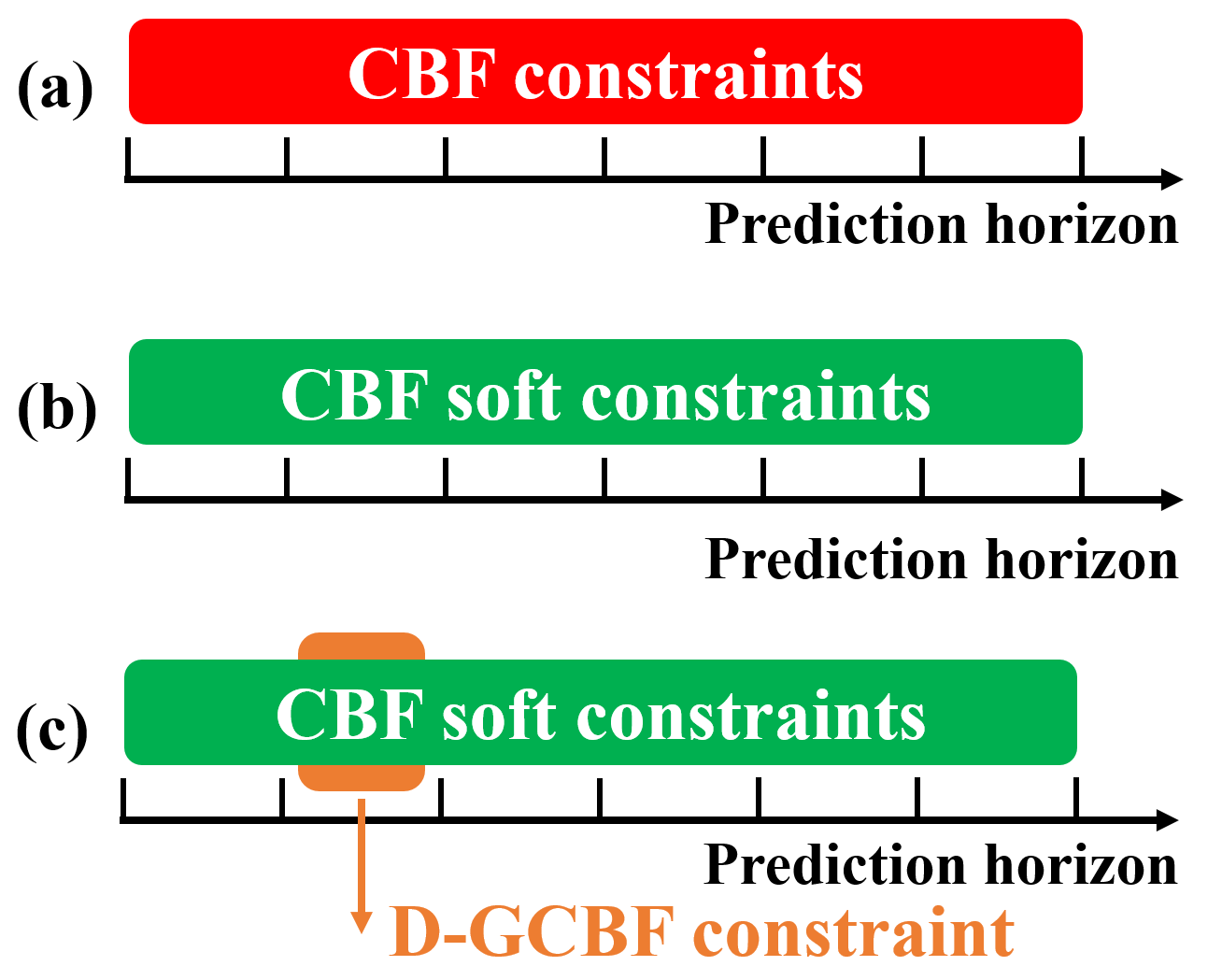}
\caption{Comparison between our obstacle avoidance constraints and those previously used. (a) shows that the CBF is imposed as hard constraints across the entire prediction horizon, and (b) represents that the CBF soft constraints functioning over this whole prediction horizon. Based on the formulation for (b), (c) introduces an additional D-GCBF constraint, which is hard and can enhance feasibility of the robot.}
\label{comparison}
\end{figure}

The reason $\eta$ is lower bounded by $\gamma$ is that we don't want the hard constraint (\ref{eq:hx3}) to be stricter than the soft constraint (\ref{our:f}). Upon integrating the D-GCBF constraint, a comparison between our obstacle avoidance constraints and those previously used is depicted in Fig. \ref{comparison}. As can be seen from Fig. \ref{comparison}, when utilizing CBF as hard constraints, it can potentially lead to infeasibility issues. By converting these into soft constraints, the optimization problem will always remain solvable. Simultaneously, the addition of a one-step D-GCBF hard constraint could possibly result in solution failure as well, but it is comparatively less stringent. The complete optimization problem, i.e., SCMPC-CBF with D-GCBF, is constructed as follows
\begin{subequations}
    \label{eq:ours}
    \begin{align}
        \min_{\mathbf{u}_{t+k|t},\mathbf{\zeta}_{t+k|t}}\ &p(\mathbf{x}_{t+N|t})+\sum^{N-1}_{k=0}q(\mathbf{x}_{t+k|t},\mathbf{u}_{t+k|t}) \notag\\
        &+\alpha \sum^{N-1}_{k=0}\|\mathbf{\zeta}_{t+k|t}\| \label{ours:a}\\
        \mathrm{s.t.}\ &\mathrm{for}\ \mathrm{all}\ k=0,\dots,N-1: \notag\\
        &\mathbf{x}_{t+k+1|t}=f(\mathbf{x}_{t+k|t},\mathbf{u}_{t+k|t}), \label{ours:b}\\
        &\mathbf{x}_{t+k|t}\in\mathcal{X}, \label{ours:c}\\
        &\mathbf{u}_{t+k|t}\in\mathcal{U}, \label{ours:d}\\
        &\mathbf{x}_{t|t}=\mathbf{x}_t, \label{ours:e}\\
        &\mathbf{x}_{t+k|t}\in\mathcal{X}(\mathbf{\zeta}_{t+k|t}),\mathbf{\zeta}_{t+k|t}\geq0, \label{ours:f} \\
        &h_i(\mathbf{x}_{t+d|t},\mathbf{o}^i_{t+d|t})\geq(1-\eta)^d h_i(\mathbf{x}_{t|t},\mathbf{o}^i_{t|t}). \label{ours:g}
    \end{align}
\end{subequations}
The addition of the constraint (\ref{ours:g}) can guarantee safety for static obstacles, which is stated in the following theorem.

\begin{thm}\label{thm:gcbf}
    For a relative-degree $d$ state constraints $h(\mathbf{x})\geq 0$ and the corresponding safe set (\ref{ss:a}), 
    assume the system satisfies $h(\mathbf{x}_{t+j})\geq 0$ for all $j\in \{1, \ldots, d-1\}$. Then
    by solving (\ref{eq:ours}) at time $t$, if feasible solutions $\mathbf{u}_{t+k|t}, k=0, \ldots, N-1$ and $\zeta_{t+k|t}, k=1, \ldots, N$ to the problem (\ref{eq:ours}) can be found, then in the next time steps, the state can be guaranteed to be within the safe set (\ref{ss:a}).
\end{thm}
\begin{proof}[Proof]
According to \cite{ma2021feasibility}, if the system satisfies $h(\mathbf{x}_{t+j})\geq 0$ for all $j\in \{1, \ldots, d-1\}$, then the set (\ref{ss:a})
defines a forward invariant safe set. Besides, as there is some control policy $\mathbf{u}_{t+k|t}$ such that (\ref{ours:g}) holds, the state is guaranteed to be with the safe set (\ref{ss:a}), i.e., at time $t+1$, we have
\[
h(\mathbf{x}_{t+1})\geq 0,
\]
i.e., $\mathbf{x}_{t+i}\in \mathcal{C}$.
\end{proof}

Therefore, for a fixed CBF $h(\mathbf{x})$, if the problem (\ref{eq:ours}) is feasible, the system is always safe, and the system state $\mathbf{x}$ is always in the safe set $\mathcal{C}$.

\begin{rem}
It is noted that D-GCBF simplifies multistep constraints into a single step, reducing computational complexity and enhancing feasibility. However, a one-step constraint may only partially ensure system safety\cite{zeng2021enhancing}, as the forward invariance of the set $\mathcal{C}$ (4a) is guaranteed provided that $h(\mathbf{x}_{t+j})\geq 0$ holds for $j \in \{1, \ldots, d-1\}$, which are not included in the constraints of (\ref{eq:ours}). Hence, solving (\ref{eq:ours}) alone in general cannot guarantee the safety of the system. Therefore, we use D-GCBF as a single-step safeguard to reinforce the system's safety after obtaining a potential safe trajectory through the optimization problem (\ref{eq:our}).
\end{rem}

\begin{rem}
In practice, the moving obstacles $\mathbf{o}^i, i=1, \ldots, N^o$ cause the changing of the CBF $h_i(\mathbf{x},\mathbf{o}^i)$; thus, the safe set also evolves with time. And if we can show that 
\begin{equation}\label{eq:Csubset}
\mathcal{C}_t\subset \mathcal{C}_{t+1}, \forall t
\end{equation}
where $\mathcal{C}_t$ is the safe set at time $t$, then according to Theorem \ref{thm:gcbf} we have
\[
\mathbf{x}_{t+1}\in \mathcal{C}_t\subset \mathcal{C}_{t+1},
\]
i.e., the system is always safe at future times. 

Although (\ref{eq:Csubset}) is not always satisfied, and the safety of future times can not be guaranteed, we find in experiments that D-GCBF actually enhances safety.
\end{rem}

It is noted that the hard constraint (\ref{ours:g}) can also bring the problem of infeasibility. However, we can see that the constraint (\ref{ours:g}) is far weaker than (\ref{mpcdcbf:f}), as (\ref{mpcdcbf:f}) is imposed on the entire prediction horizon, and besides, $\eta \geq \gamma$, hence the chance of infeasibility of (\ref{eq:ours}) is far less than that of (\ref{eq:mpcdcbf}). To avoid possible damage, when (\ref{eq:ours}) is infeasible, we ensure that the robot comes to a complete stop by activating the brakes instead of merely setting the control input to zero. The complete algorithm is presented as Algorithm \ref{alg:alg1}.

\begin{algorithm}[H]
    \renewcommand{\algorithmicrequire}{\textbf{Input:}}
	\renewcommand{\algorithmicensure}{\textbf{Output:}}
    \caption{SCMPC-CBF with D-GCBF}\label{alg:alg1}
    \begin{algorithmic}[1]
        \REQUIRE
        Initial state $\mathbf{x}(t)$, state constraints $\mathcal{X}$, input constraints $\mathcal{U}$, system dynamic (\ref{eq:rbsys}), obstacles state $\mathbf{o}^i(t)$, system dynamic (\ref{eq:mosys}) of obstacle, goal state $\mathbf{x}_{goal}$.
        \ENSURE
        Optimal control $\mathbf{u}(t)$.
        \STATE $\mathbf{x}_t=\mathbf{x}(t)$.
        \STATE $\mathbf{o}^i_t=\mathbf{o}^i(t)$.
        \STATE Solve (\ref{eq:ours}).
        \IF{(\ref{eq:ours}) is solved successfully}
            \RETURN $\mathbf{u}(t) = \mathbf{u}^*_{t|t}\in\mathbf{u}^*_{t+k|t}$
        \ELSE
            \STATE Activate the braking mechanism on the robot.
        \ENDIF
    \end{algorithmic}
\label{alg1}
\end{algorithm}

\section{EXPERIMENTS} \label{sec4}
In this section, experiments in simulation environments and real scenarios are conducted to verify the effectiveness of our work.
\subsection{Simulation Setup}
All controllers are implemented in Python with Casadi\cite{andersson2019casadi} as modeling language, solved with IPOPT \cite{biegler2009large}. The simulation experiments were conducted on a computer running Ubuntu 20.04, which used an Intel Core i5-12490f processor with 16 GB RAM. 

Following \cite{chen2019crowd}, each agent must stay within a 10m×10m two-dimensional space. The simulated pedestrians are controlled by ORCA \cite{van2011reciprocal}, and their initial positions are randomly sampled on a circle with a radius of 4m, and their target positions are on the other side of the same circle. The robot has the same radius of 0.3m and the same preferred speed of 1m/s as the pedestrian. And the simulation time step is 0.2s.

In order to fully evaluate the safety and effectiveness of the proposed algorithm, we set the robot to be invisible to other humans. That is, the simulated human only reacts to humans and turns a blind eye to the robot. All controller algorithms are evaluated using 500 random test cases with five pedestrians.

\subsection{Quantitative Evaluation}
All controller performances are evaluated under the same prediction horizon $N$ and the same form of stage and terminal cost. The motion model (\ref{eq:mosys}) of all obstacles is approximated by a linear model. We use the following indicators to compare the five controllers: S (the rate of the robot reaching its destination without collision), C (the rate of the robot colliding with moving obstacles), T (the robot's average navigation time in seconds), FS (the average number of solution failures), ST (the robot's average solution time in milliseconds).
\subsubsection{Double-integrator system} Consider the robot's motion model (\ref{eq:rbsys}) as a discrete-time linear double-integrator system,
\begin{equation}
    \mathbf{x}_{k+1} = A\mathbf{x}_{k} + B\mathbf{u}_{k},
\end{equation}
where $\mathbf{x}=[x,y,v_x,v_y]^T$ and $\mathbf{u}=[a_x,a_y]^T$ represent position $(x,y)$, velocity $(v_x,v_y)$ and acceleration $(a_x,a_y)$, respectively. We set $\epsilon=0.2$ in MPC-DC (\ref{eq:mpcdc}) and $\epsilon=0$ in other methods. The results are shown in Table \ref{tab:table1}.

\begin{table}[!t]
\caption{Quantitative Results With Double-integrator System\label{tab:table1}}
\centering
\begin{tabular}{|c|c|c|c|c|c|c|}
\hline
Controller & $\gamma$ & S $\uparrow$ & C $\downarrow$ & T & FS & ST\\
\hline
ORCA \cite{van2011reciprocal} & - & 0.470 & 0.526 & 11.04 & - & -\\
\hline
MPC-DC(\ref{eq:mpcdc}) & - & 0.362 & 0.636 & 12.97 & 6.486 & 45.11\\
\hline
\multirow{3}{*}{MPC-D-CBF(\ref{eq:mpcdcbf})}
 & 0.08 & 0.736 & 0.262 & 16.73 & 21.734 & 42.47\\
 & 0.10 & 0.684 & 0.316 & 14.91 & 16.602 & 42.78\\
 & 0.12 & 0.624 & 0.376 & 13.80 & 13.544 & 44.12\\
\hline
\multirow{3}{*}{SCMPC-CBF(\ref{eq:our})}
 & 0.08 & 0.952 & 0.048 & 13.45 & 0 & 53.01\\
 & 0.10 & 0.776 & 0.224 & 12.19 & 0 & 53.48\\
 & 0.12 & 0.640 & 0.360 & 11.61 & 0 & 53.33\\
\hline
\multirow{3}{*}{Ours(\ref{eq:ours})}
 & 0.08 & \textbf{0.996} & \textbf{0.004} & 14.35 & 0.374 & 55.08\\
 & 0.10 & 0.966 & 0.034 & 13.61 & 0.794 & 55.76\\
 & 0.12 & 0.954 & 0.046 & 13.34 & 1.242 & 55.51\\
\hline
\end{tabular}
\end{table}

The results show that in a dynamic environment, short-term ORCA performs poorly in the invisible setting due to a violation of the reciprocal assumption. MPC-DC still performs poorly even if an additional safety margin $\epsilon$ is added. This is because it cannot avoid dynamic obstacles in advance and easily enters high-risk areas. The safety of MPC-D-CBF will improve as $\gamma$ decreases, but correspondingly, the navigation efficiency will decrease due to tighter hard constraints. Robots driven by soft constraints will more actively explore new paths in crowded areas, and single-step safety constraints will also enhance the safety of robot navigation. Although the introduction of slack variables may slightly increase the complexity of the solution, it can prevent the extra time consumption caused by solution failure. As shown in Fig. \ref{holo}, we compared the navigation paths of these controllers in the 330th test case. 

\begin{figure*}[!t]
\centering
\subfloat[ORCA]{\includegraphics[width=0.2\textwidth]{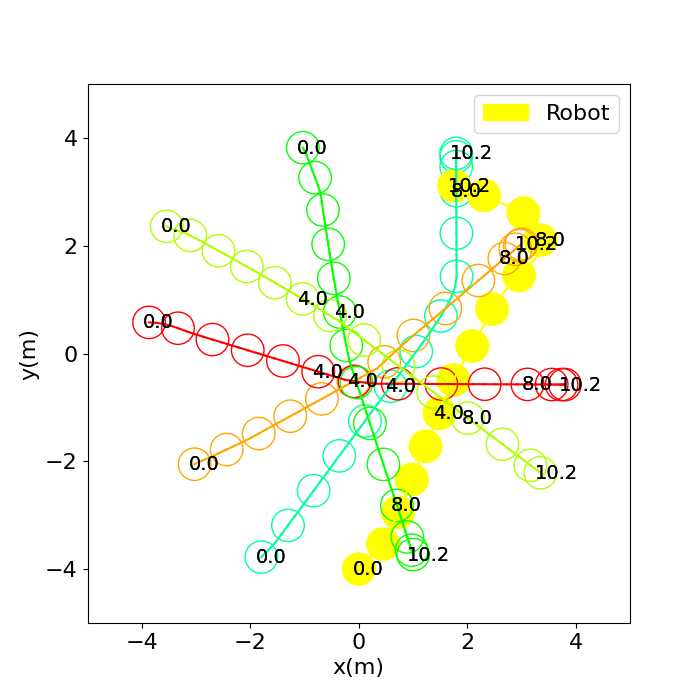}\label{ORCA_case}}
\hfil
\subfloat[MPC-DC]{\includegraphics[width=0.2\textwidth]{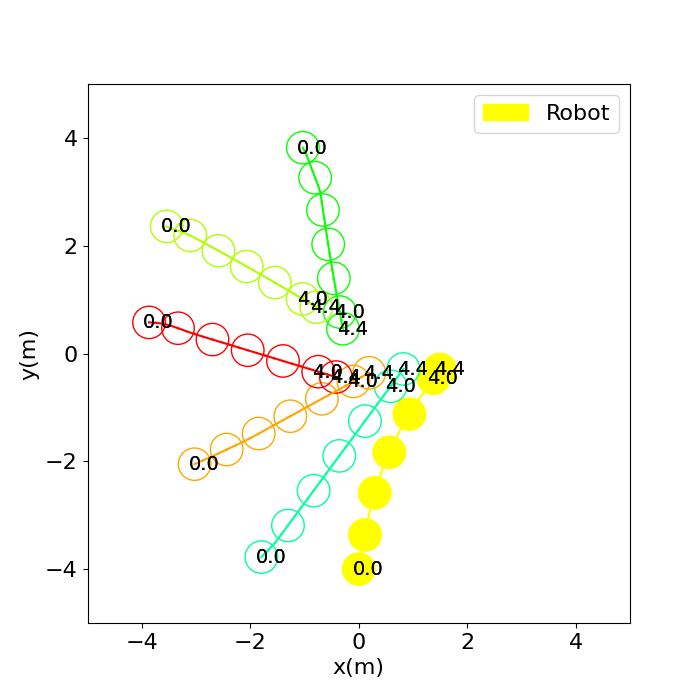}\label{MPC-DC_case}}
\hfil
\subfloat[MPC-D-CBF]{\includegraphics[width=0.2\textwidth]{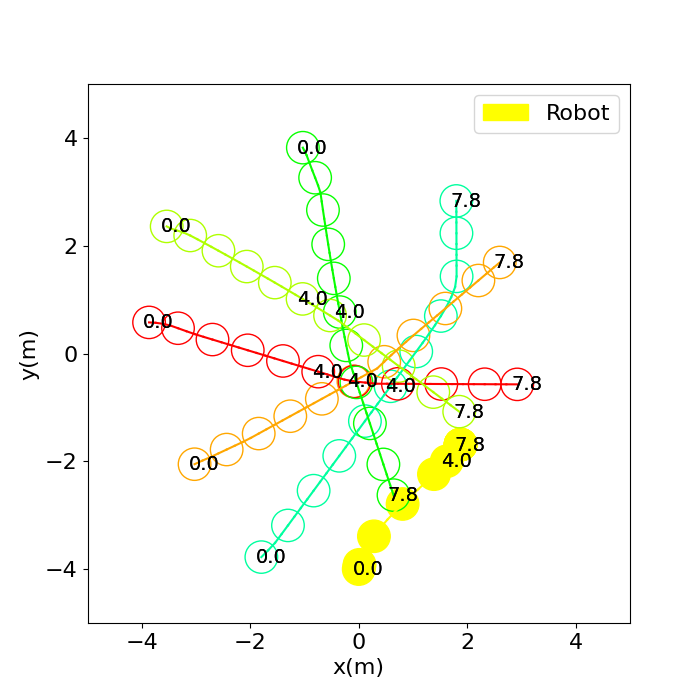}\label{MPCCBF_case}}
\hfil
\subfloat[SCMPC-CBF]{\includegraphics[width=0.2\textwidth]{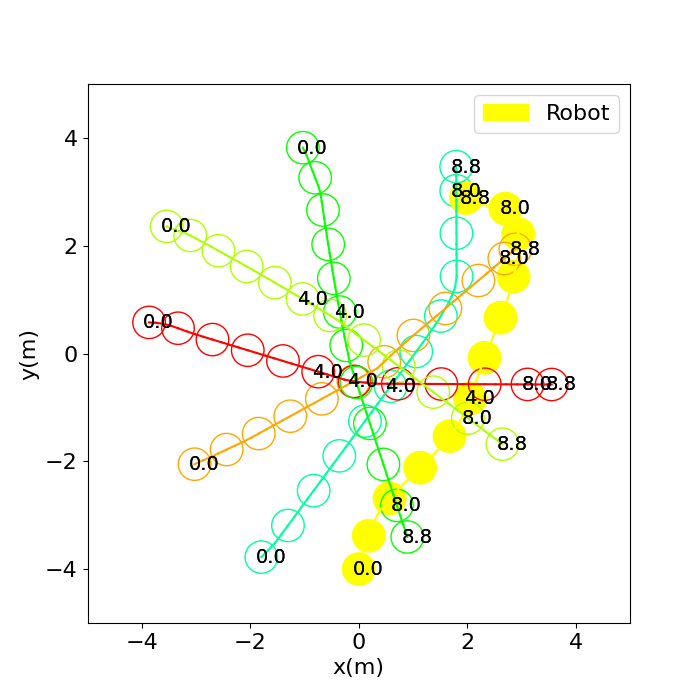}\label{SCMPC_case}}
\hfil
\subfloat[Ours]{\includegraphics[width=0.2\textwidth]{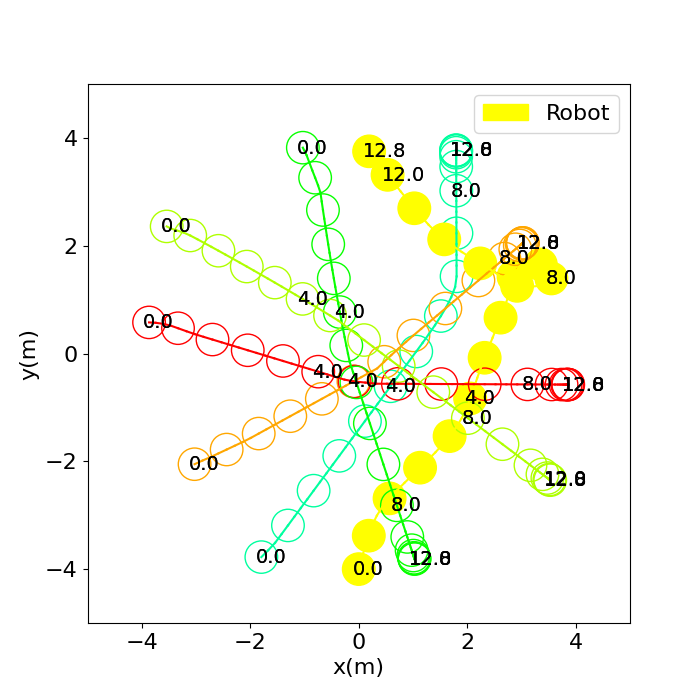}\label{OUR_case}}
\caption{Trajectories of different controllers with double-integrator kinematics system. The circles in the picture are the locations of the agents at the labeled time. In the 330th test case, our method can make the robot successfully reach the destination. More test results are shown in Table \ref{tab:table1}.}
\label{holo}
\end{figure*}

\subsubsection{Unicycle system}  Consider the robot's motion model (\ref{eq:rbsys}) as a discrete-time nonlinear unicycle system,
\begin{equation}
    \mathbf{x}_{k+1} = f(\mathbf{x}_{k},\mathbf{u}_{k}),
\end{equation}
where $\mathbf{x}=[x,y,\theta]^T$ and $\mathbf{u}=[v,\omega]^T$ represent position $(x,y)$, heading angle $\theta$, line speed $v$ and angular velocity $\omega$, respectively. VO-based methods such as ORCA are suitable for robots that can move in any direction but are not suitable for robots with non-holonomic kinematics\cite{alonso2013optimal}. Therefore, we do not compare with ORCA. We set $\epsilon=0.2$ in MPC-DC (\ref{eq:mpcdc}) and $\epsilon=0$ in other methods. The results are shown in Table \ref{tab:table2}.

\begin{table}[!t]
\caption{Quantitative Results With Unicycle System\label{tab:table2}}
\centering
\begin{tabular}{|c|c|c|c|c|c|c|}
\hline
Controller & $\gamma$ & S $\uparrow$ & C $\downarrow$ & T & FS & ST\\
\hline
MPC-DC(\ref{eq:mpcdc}) & - & 0.186 & 0.808 & 10.87 & 6.116 & 50.78\\
\hline
\multirow{3}{*}{MPC-D-CBF(\ref{eq:mpcdcbf})}
 & 0.08 & 0.848 & 0.152 & 14.40 & 8.862 & 49.07\\
 & 0.10 & 0.764 & 0.236 & 13.27 & 7.756 & 49.13\\
 & 0.12 & 0.724 & 0.276 & 12.58 & 6.976 & 49.27\\
\hline
\multirow{3}{*}{SCMPC-CBF(\ref{eq:our})}
 & 0.08 & 0.980 & 0.020 & 14.44 & 0 & 60.59\\
 & 0.10 & 0.936 & 0.064 & 13.27 & 0 & 60.75\\
 & 0.12 & 0.900 & 0.100 & 12.68 & 0 & 60.83\\
\hline
\multirow{3}{*}{Ours(\ref{eq:ours})}
 & 0.08 & \textbf{0.982} & \textbf{0.018} & 14.56 & 0.032 & 61.54\\
 & 0.10 & 0.960 & 0.040 & 13.34 & 0.128 & 62.37\\
 & 0.12 & 0.902 & 0.098 & 12.71 & 0.388 & 61.03\\
\hline
\end{tabular}
\end{table}

The results demonstrate that our method achieves a high success rate despite the reduced action space of non-holonomic kinematics systems. Given the underactuated characteristics of non-holonomic mobile robots \cite{pappalardo2019forward} (3 degrees of freedom $(x,y,\theta)$, 2 controls $(v,\omega)$), the obstacle avoidance capability of these robots is inevitably limited. As a result, the superiority of our improved method over other methods is not so apparent as that in simulation experiments for double-integrator system. Similar to the previous simulation experiment for double-integrator system, as the value of $\gamma$ decreases, the system's safety during navigation improves. However, it also yields a higher probability of solution failure. The increase in average solution time can be attributed to the system's nonlinearity. The simulation experiment conducted using the unicycle model lays the groundwork for subsequent experiments in real-world scenarios. Fig. \ref{nonholo} illustrates a comparison of the navigation paths of these controllers in the 116th test case. Our code and further examples can be found at \url{http://https://github.com/Zetao-Lu/CrowdNav_MPCCBF}

\begin{figure}[!t]
\centering
\subfloat[MPC-DC]{\includegraphics[width=0.2\textwidth]{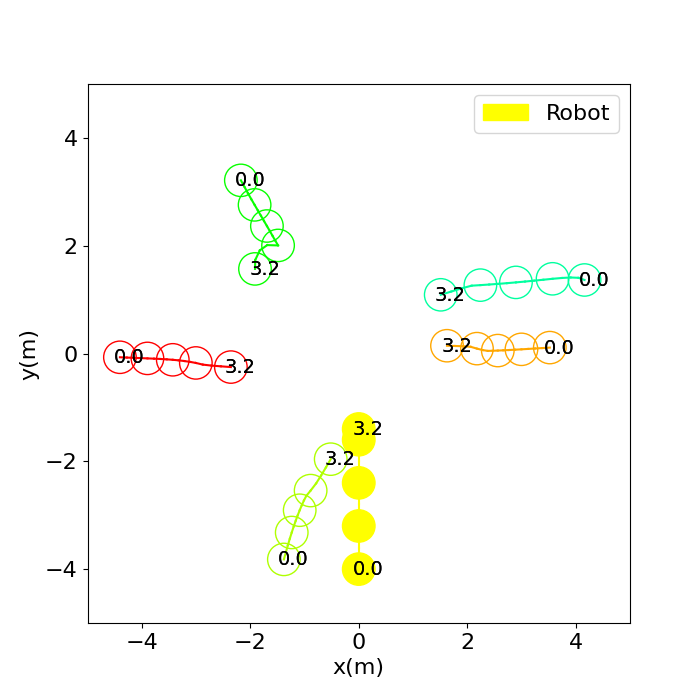}\label{MPCDCnh_case}}
\hfil
\subfloat[MPC-D-CBF]{\includegraphics[width=0.2\textwidth]{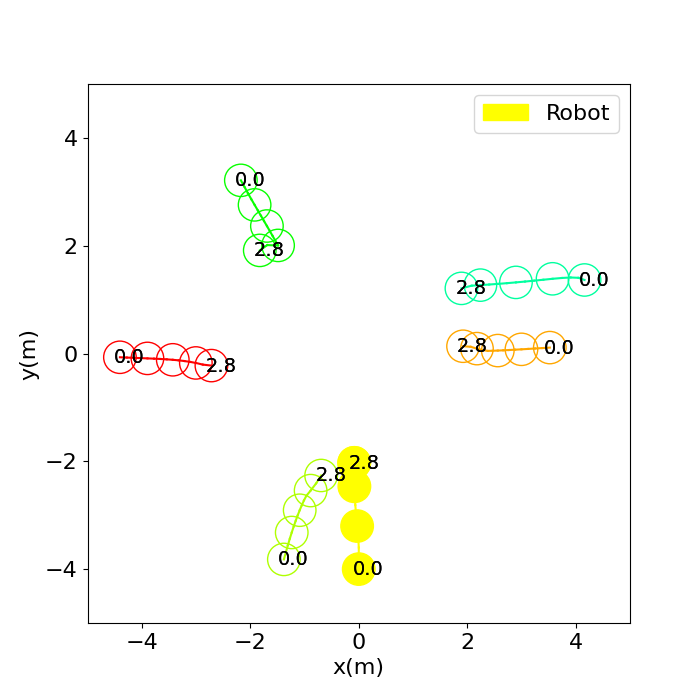}\label{MPCCBFnh_case}}
\hfil
\subfloat[SCMPC-CBF]{\includegraphics[width=0.2\textwidth]{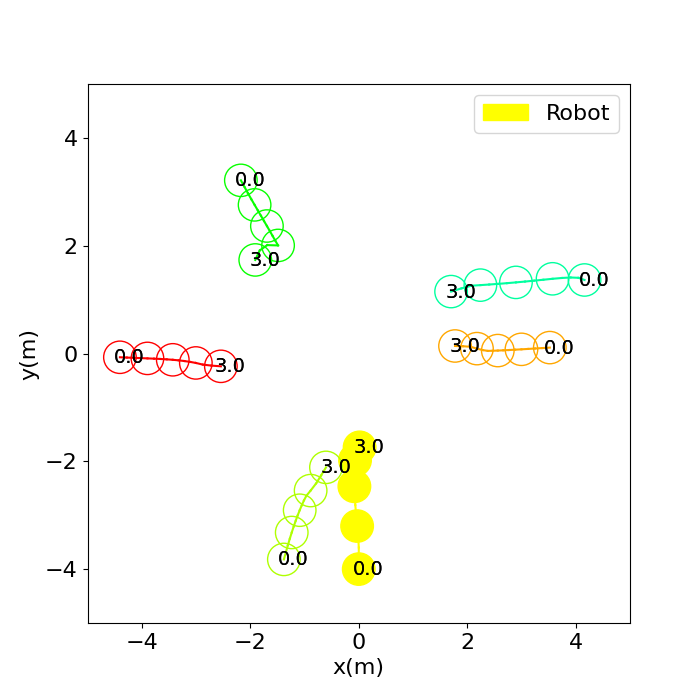}\label{SCMPCnh_case}}
\hfil
\subfloat[Ours]{\includegraphics[width=0.2\textwidth]{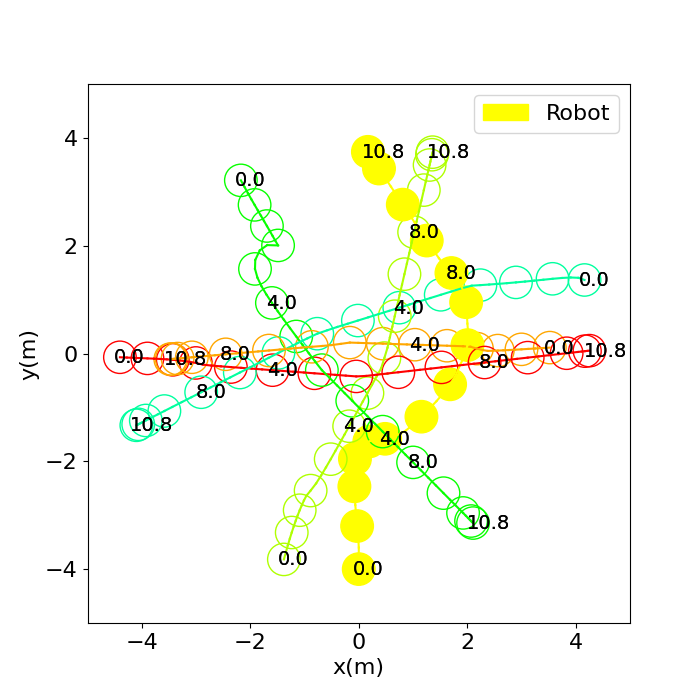}\label{OURnh_case}}
\caption{Trajectories of different controllers with unicycle system. The circles in the picture are the locations of the agents at the labeled time. In the 116th test case, our method can make the robot successfully reach the destination. More test results are shown in Table \ref{tab:table2}.}
\label{nonholo}
\end{figure}

\subsection{Real-world Experiments}
In real-world experiments, we deployed our method on an MR1000 robot. We set the sensing range of the 64-line lidar to be 8 meters and used PointPillars\cite{lang2019pointpillars} for pedestrian detection within this range. Additionally, we utilized AB3DMOT\cite{weng20203d} to track the detection results and estimate their relative positions and velocities. After building a map of real-world environment, we estimated the robot's state using the AMCL package in ROS and Kalman filter. As shown in Fig. \ref{real}, in real-world experiments, the MR1000 robot successfully navigated to the target location without colliding with pedestrians by utilizing our method as a local planning controller. The results demonstrate the successful transferability of our method from simulation to real robots as a local planning module, ensuring safety for the robot.

\begin{figure*}[!t]
\centering
\subfloat[start time]{\includegraphics[width=0.3\textwidth]{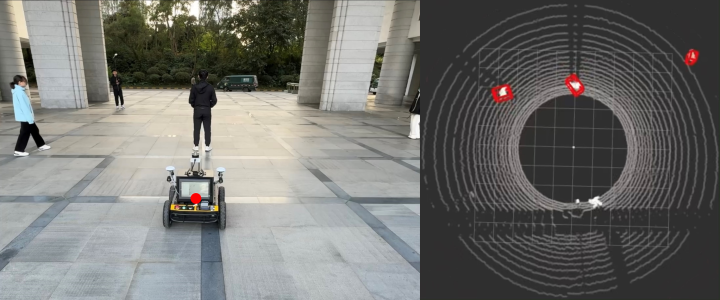}\label{0s}}
\hfil
\subfloat[after 4 seconds]{\includegraphics[width=0.3\textwidth]{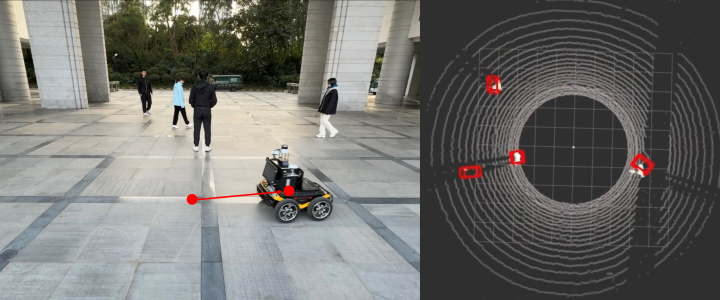}\label{4s}}
\hfil
\subfloat[after 8 seconds]{\includegraphics[width=0.3\textwidth]{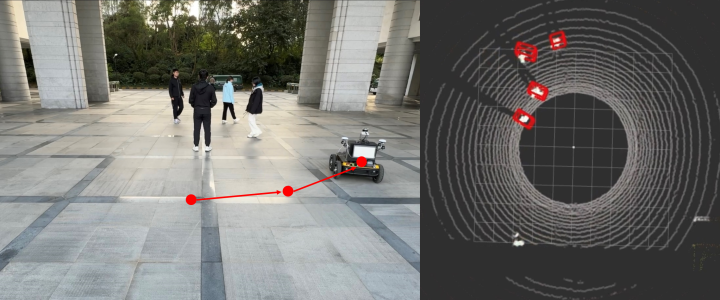}\label{8s}}
\hfil
\subfloat[after 12 seconds]{\includegraphics[width=0.3\textwidth]{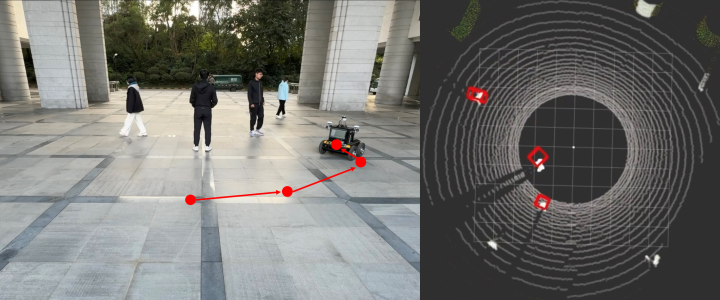}\label{12s}}
\hfil
\subfloat[after 16 seconds]{\includegraphics[width=0.3\textwidth]{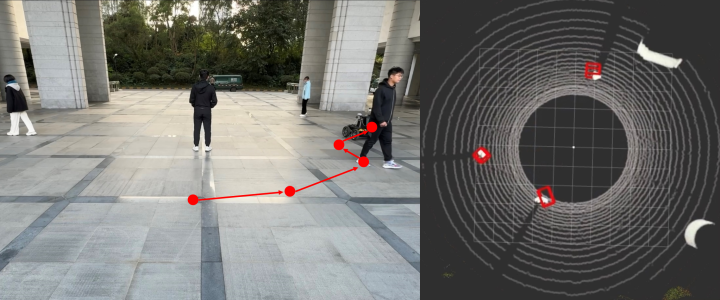}\label{16s}}
\hfil
\subfloat[end time]{\includegraphics[width=0.3\textwidth]{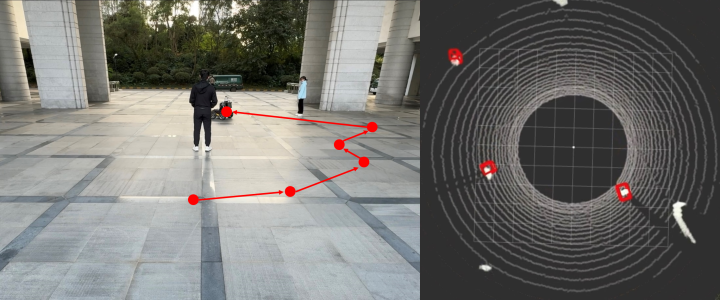}\label{20s}}
\caption{Real-world experiments have confirmed the effectiveness of the proposed collision avoidance method. The robot's navigation sequence is illustrated by these six sub-images. When pedestrians enter the sensing range of the lidar, the robot detects them. As depicted on the right side of the sub-image, perceived pedestrians are depicted by red boxes.}
\label{real}
\end{figure*}

\section{CONCLUSIONS} \label{sec5}
In this paper, we propose a new MPC framework that integrates the CBF to address the challenge of obstacle avoidance in dynamic environments while avoiding the infeasibility problem caused by hard constraints acting on the entire predictive horizon. Additionally, we design a D-GCBF based on the relative degree of constraints to the system, enhancing the robot's obstacle avoidance capability under soft constraints by employing single-step safety constraints. Experimental results demonstrate that our method achieves a higher navigation success rate, lower collision rate, and lower solution failure probability compared to other baseline methods. Furthermore, we deploy the method as a local planning controller on the MR1000 robot using the ROS platform and validate the effectiveness of our approach in real-world environments.

\bibliographystyle{IEEEtran}
\bibliography{ref}

\vfill
\end{document}